\documentclass[12pt]{article}

\usepackage{fullpage}
\usepackage{amsmath,amssymb,amsthm,graphics,todonotes,paralist}
\usepackage{xspace}
\usepackage{algorithm,algorithmic} 

\usepackage{hyperref}
\usepackage[round]{natbib}

\newtheorem{theorem}{Theorem}

\newtheorem{reduction}{Reduction}

\newcommand{\R}{\mathbb{R}}  

\newcommand{\w}{\mathbf{w}}  
\newcommand{\x}{\mathbf{x}}  
\newcommand{\z}{\mathbf{z}}  
\newcommand{\I}{\mathbf{I}}  
\newcommand{\X}{\mathbf{X}}  
\newcommand{\one}{\mathbf{1}}  
\newcommand{\zero}{\mathbf{0}}  

\newcommand{\M}{\mathbf{M}}
\newcommand{\hx}{\mathbf{\hat{x}}}
\newcommand{\NP}{\textsf{NP}\xspace}

\newcommand{\BPP}{\textsf{BPP}\xspace}
\newcommand{\E}{\ensuremath{\mathbb{E}}}
\def\setcover{\textsf{Set Cover}\xspace}
\def\CNF{\textsf{3CNF}\xspace}
\def\SAT{\textsf{SAT}\xspace}

\def\algosr{\textsf{Alg}_\textsf{OSR}}
\def\algsat{\textsf{Alg}_\textsf{SAT}}

\def\hy{\hat{y}}
\def\hS{\hat{S}}

\DeclareMathOperator{\poly}{poly}

\title{Online Sparse Linear Regression}

\author{Dean Foster \\
Amazon \\
\texttt{\small dean@foster.net}
\and
Satyen Kale \\
Yahoo Research\\
\texttt{\small satyen@yahoo-inc.com}
\and
Howard Karloff \\
Goldman Sachs \\
\texttt{\small howard@cc.gatech.edu}}
\date{}

\begin{document}

\maketitle


\begin{abstract}
We consider the online sparse linear regression problem, which is the problem of sequentially making predictions observing only a limited number of features in each round, to minimize regret with respect to the best sparse linear regressor, where prediction accuracy is measured by square loss. We give an {\em inefficient} algorithm that obtains regret bounded by $\tilde{O}(\sqrt{T})$ after $T$ prediction rounds. We complement this result by showing that no algorithm running in polynomial time per iteration can achieve regret bounded by $O(T^{1-\delta})$ for any constant $\delta > 0$ unless $\NP \subseteq \BPP$. This computational hardness result resolves an open problem presented in COLT 2014~\citep{open-problem} and also posed by \citet{andras}. This hardness result holds even if the algorithm is allowed to access more features than the best sparse linear regressor up to a logarithmic factor in the dimension.
\end{abstract}

\section{Introduction}

In various real-world scenarios, features for examples are constructed by running some computationally expensive algorithms. With resource constraints, it is essential to be able to make predictions with only a limited number of features computed per example. One example of this scenario, from \citep{CBSS}, is medical diagnosis of a disease, in
which each feature corresponds to a medical test that the patient in question can undergo. Evidently, it is undesirable to subject a patient to a battery of medical tests, for medical as well as cost reasons. Another example from the same paper is a search engine, where a ranking of web pages must be generated for each incoming user query and the limited amount of time allowed to answer a query imposes restrictions on the number of attributes that can be evaluated in the process. In both of these problems, predictions need to be made sequentially as patients or search queries arrive online, learning a good model in the process.

In this paper, we model the problem of prediction with limited access to features in the most natural and basic manner as an online sparse linear regression problem. In this problem, an online learner makes real-valued predictions for the labels of examples arriving sequentially over a number of rounds. Each example has $d$ features that can be potentially accessed by the learner. However, in each round, the learner is restricted to choosing an arbitrary subset of features of size at most $k$, a budget parameter. The learner then acquires the values of the subset of features, and then makes its prediction, at which point the true label of the example is revealed to the learner. The learner suffers a loss for making an incorrect prediction (for simplicity, we use square loss in this paper). The goal of the learner is to make predictions with total loss comparable to the loss of the best sparse linear regressor with a bounded norm, where the term {\em sparse} refers to the fact that the linear regressor has nonzero weights on at most $k$ features. To measure the performance of the online learner, we use the standard notion of {\em regret}, which is the difference between the total loss of the online learner and the total loss of the best sparse linear regressor.

While regret is the primary performance metric, we are also interested in {\em efficiency} of the online learner. Ideally, we desire an online learning algorithm that minimizes regret while making predictions {\em efficiently}, i.e., in polynomial time (as a function of $d$ and $T$). In this paper, we prove that this goal is impossible unless there is a randomized polynomial-time algorithm for deciding satisfiability of \CNF formulas, the canonical \NP-hard problem. This computational hardness result resolves open problems from \citep{open-problem} and \citep{andras}. In fact, the computational hardness persists even if the online learner is given the additional flexibility of choosing $k' = D \log(d) k$ features for any constant $D > 0$. In light of this result, in this paper we also give an {\em inefficient} algorithm for the problem which queries $k' \geq k + 2$ features in each round, that runs in $O({d \choose k} k')$ time per round,
and that obtains regret bounded by $O(\tfrac{d^2}{(k' - k)^2}\sqrt{k \log(d)T})$.

\section{Related Work and Known Results}

A related setting is attribute-efficient learning \citep{CBSS,HK,KS}. This is a batch learning problem in which the examples are generated i.i.d., and the goal is to simply output a linear regressor using only a limited number of features per example with bounded excess risk compared to the optimal linear regressor, when given {\em full access} to the features at test time. While the aforementioned papers give efficient, near-optimal algorithms for this problem, these algorithms do not work in the online sparse regression setting in
which we are interested because here we are required to make predictions only using a limited number of features.

In \citep{open-problem}, a simple algorithm has been suggested, which is  based on running a bandit algorithm in which the actions correspond to selecting one of ${d \choose k}$ subsets of coordinates of size $k$ at regular intervals, and within each interval, running an online regression algorithm (such as the Online Newton-Step algorithm of \citet{HKKA}) over the $k$ coordinates chosen by the bandit algorithm. This algorithm, with the right choice of interval lengths, has a regret bound of $O(k^2d^{k/3}T^{2/3}\log(T/d))$. The algorithm has exponential dependence on $k$ both in running time and the regret. Also, \citet{open-problem} sketches a different algorithm with performance guarantees similar to the algorithm presented in this paper; our work builds upon that sketch and gives tighter regret bounds.

\citet{andras} consider a very closely related setting (called {\em online probing}) in which features and labels may be obtained by the learner at some cost (which may be different for different features), and this cost is factored into the loss of the learner. In the special case of their setting corresponding to the problem considered here, they given an algorithm, \textsc{LQDExp3}, which relies on discretizing all $k$-sparse weight vectors and running an exponential-weights experts algorithm on the resulting set with stochastic loss estimators, obtaining a $O(\sqrt{dT})$ regret bound. However the running time of their algorithm is prohibitive: $O((dT)^{O(k)})$ time per iteration. In the same paper, they pose the open problem of finding a computationally efficient no-regret algorithm for the problem. The hardness result in this paper resolves this open problem.

On the computational hardness side, it is known that it is \NP-hard to compute the optimal sparse linear regressor \citep{FKT, natarajan}. The hardness result in this paper is in fact inspired by the work of \citet{FKT}, who proved that it is computationally hard to find even an approximately optimal sparse linear regressor for an ordinary least squares regression problem given a batch of labeled data. While these results imply that it is hard to {\em properly}\footnote{{\em Proper} learning means finding the optimal sparse linear regressor, whereas {\em improper} learning means finding an arbitrary predictor with performance comparable to that
of the optimal sparse linear regressor.} solve the offline problem, in the online setting we allow {\em improper} learning, and hence these prior results don't yield hardness results for the online problem considered in this paper.

\section{Notation and Setup}

We use the notation $[d] = \{1, 2, \ldots, d\}$ to refer to the coordinates. All vectors in this paper are in $\R^d$, and all matrices in $\R^{d \times d}$. For a subset $S$ of $[d]$, and a vector $\x$, we use the notation $\x(S)$ to denote the projection of $\x$ on the coordinates indexed by $S$. We also use the notation $\I_S$ to denote the diagonal matrix which has ones in the coordinates indexed by $S$ and zeros elsewhere: this is the identity matrix on the subspace of $\R^d$ induced by the coordinates in $S$, as well as the projection matrix for this subspace. We use the notation $\|\cdot\|$ to denote the $\ell_2$ norm in $\R^d$ and $\|\cdot\|_0$ to denote the zero ``norm,'' i.e., the number of nonzero coordinates.

We consider a prediction problem in which the examples are vectors in $\R^d$ with $\ell_2$ norm  bounded by $1$, and labels are in the range $[-1, 1]$. We use square loss to measure the accuracy of a prediction: i.e., for a labeled example $(\x, y) \in \R^d \times [-1, 1]$, the loss of a prediction $\hy$ is $(\hy - y)^2$. The learner's task is to make predictions online as examples arrive one by one based on observing only $k$ out of $d$ features of the learner's choosing on any example (the learner is allowed to choose different subsets of features to observe in each round). The learner's goal is to minimize regret relative to the best $k$-sparse linear regressor whose $\ell_2$ norm is bounded by $1$.

Formally, for $t = 1, 2, \ldots, T$, the learner:
\begin{enumerate}
	\item selects a subset $S_t \subseteq [d]$ of size at most $k$,
	\item observes $\x_t(S_t)$, i.e., the values of the features of $\x_t$ restricted to the subset $S_t$,
	\item makes a prediction $\hy_t \in [-1, 1]$,
	\item observes the true label $y_t$, and suffers loss $(\hy_t - y_t)^2$.
\end{enumerate}
Define regret of the learner as
\[\text{Regret}\ :=\ \sum_{t=1}^T (\hy_t - y_t)^2 - \min_{\w:\ \|\w\|_0 \leq k,\ \|\w\| \leq 1} \sum_{t=1}^T (\w \cdot \x_t - y_t)^2. \]
In case $\hy_t$ is chosen using randomization, we consider expected regret instead. 

Given the \NP-hardness of computing the optimal $k$-sparse linear regressor \citep{FKT, natarajan}, we also consider a variant of the problem which gives the learner more flexibility than the comparator: the learner is allowed to choose $k' \geq k$ coordinates to query in each round. The definition of the regret remains the same. We call this the {\em $(k, k', d)$-online sparse regression problem}.

We are interested in the following two goals\footnote{In this paper, we use the $\poly(\cdot)$ notation to denote a polynomially-bounded function of its arguments.}:
\begin{enumerate}
	\item (No Regret) Make predictions $\hy_t$ so that regret is bounded by $\poly(d) T^{1-\delta}$ for some $\delta > 0$.

	\item (Efficiency) Make these predictions efficiently, i.e., in $\poly(d, T)$ time per iteration.
\end{enumerate}

In this paper, we show it is possible to get an {\em inefficient} no-regret algorithm for the online sparse regression problem. Complementing this result, we also show that an {\em efficient} no-regret algorithm cannot exist, assuming the standard hardness assumption that $\NP \not\subseteq \BPP$. 

\section{Upper bound}
\def\etaH{\eta_\textsc{Hedge}}
\def\etaGD{\eta_\textsc{SGD}}

In this section we give an {\em inefficient} algorithm for the $(k, k', d)$-online sparse regression problem which obtains an expected regret of $O(\frac{d^2}{(k'-k)^2}\sqrt{k \log(d)T})$. The algorithm needs $k'$ to be at least $k + 2$. It is inefficient because it maintains statistics for every subset of $[d]$ of size $k$, of which there are ${d \choose k}$. 

At a high level, the algorithm runs an experts algorithm (specifically, Hedge) treating all such subsets as experts. Each expert internally runs stochastic gradient descent only on the coordinates specified by the corresponding subset, ensuring low regret to any bounded norm parameter vector that is nonzero only on those coordinates. The Hedge algorithm ensures low regret to the best subset of coordinates, and thus the overall algorithm achieves low regret with respect to any $k$-sparse parameter vector. The necessity of using $k' \geq k + 2$ features in the algorithm is that the algorithm uses the additional $k' - k$ features to generate unbiased estimators for $\x_t\x_t^\top$ and $y_t \x_t$ in each round, which are needed to generate stochastic gradients for all the experts.
These estimators have large variance unless $k'-k$ is large. 

The pseudocode is given in Algorithm~\ref{algorithm:osr}. In the algorithm, in round $t$, the algorithm generates a distribution $D_t$ over the subsets of $[d]$ of size $k$; for any such subset $S$, we use the notation $D_t(S)$ to denote the probability of choosing the set $S$ in this distribution. We also define the function $\Pi$ on $\R^d$ to be the projection onto the unit ball, i.e., for $\w \in \R^d$, $\Pi(\w) = \w$ if $\|\w\| \leq 1$, and $\Pi(\w) = \frac{1}{\|\w\|}\w$ otherwise.

\begin{algorithm}[h]
\caption{\textsc{Algorithm for Online Sparse Regression}
\label{algorithm:osr}}
\begin{algorithmic}[1]
\STATE Define the parameters $p = \frac{k'-k}{d}$, $q = \frac{(k'-k)(k'-k-1)}{d(d-1)}$, $\etaH = q \sqrt{\frac{\ln(d)}{T}}$, and $\etaGD = q\sqrt{\frac{1}{T}}$.

\STATE Let $D_1$ be the uniform distribution over all subsets of $[d]$ of size $k$.

\STATE For every subset $S$ of $[d]$ of size $k$, let $\w_{S, 1} = \zero$, the all-zeros vector in $\R^d$.

\FOR{$t=1, 2, \ldots, T$}
	\STATE Sample a subset $\hS_t$ of $[d]$ of size $k$ from $D_t$, and a subset $R_t$ of $[d]$ of size $k' - k$ drawn uniformly at random, independently of $\hS_t$.

	\STATE Acquire $\x_t(S_t)$ for $S_t:=\hS_t \cup R_t$. 

	\STATE Make the prediction $\hy_t = \w_{\hS_t, t} \cdot \x_t$ and obtain the true label $y_t$. 

	\STATE Compute the matrix $\X_t \in \R^{d \times d}$ and the vector $\z_t \in \R^d$ defined as follows:
		\[ \X_t(i,j) = \begin{cases}
		\frac{\x_t(i)^2}{p} & \text{ if } i = j \text{ and } i \in R_t \\
		\frac{\x_t(i)\x_t(j)}{q} & \text{ if } i \neq j  \text{ and } i, j \in R_t\\
		0 & \text{ otherwise,}
		\end{cases}
		\quad
		\text{ and }
		\quad
	 	\z_t(i) = \begin{cases}
		\frac{y_t\x_t(i)}{p} & \text{ if } i\in R_t\\
		0 & \text{ otherwise,}
		\end{cases}\]

	\STATE Update the distribution over the subsets: for all subsets $S$ of $[d]$ of size $k$, let
	\[ D_{t+1}(S)\ =\ D_t(S) \exp(-\etaH(\w_{S, t}^\top \X_t \w_{S, t} - 2\z_t^\top \w_{S, t} + y_t^2))/Z_t, \]
	where $Z_t$ is the normalization factor to make $D_{t+1}$ a distribution.

	\STATE For each subset $S$ of $[d]$ of size $k$, let
	\[ \w_{S, t+1}\ =\ \Pi(\w_{S, t} - 2\etaGD \I_S(\X_t \w_{S, t} - \z_t)).\]
\ENDFOR
\end{algorithmic}
\end{algorithm}

\begin{theorem} \label{thm:regret-bound}
	There is an algorithm for the online sparse regression problem with any given parameters $(k, k', d)$ such that $k' \geq k + 2$ running in $O({d \choose k} \cdot k')$ time per iteration with $O(\frac{d^2}{(k' - k)^2}\sqrt{k \log(d)T})$ expected regret.
\end{theorem}
\begin{proof}
The algorithm is given in Algorithm~\ref{algorithm:osr}. Since the algorithm maintains a parameter vector in $\R^k$ for each subset of $[d]$ of size $k$, the running time is dominated by the time to sample from $D_t$ and update it, and the time to update the parameter vectors. The updates can be implemented in $O(k')$ time, so overall each round can be implemented in $O({d \choose k} \cdot k')$ time.

We now analyze the regret of the algorithm. Let $\E_t[\cdot]$ denote the expectation conditioned on all the randomness prior to round $t$. Then, it is easy to check,
using the fact that $k'-k\ge 2$, that the construction of $\X_t$ and $\z_t$ in Step 8 of the algorithm has the following property:
\begin{equation} \label{eq:unbiased}
	\E_t[\X_t] = \x_t\x_t^\top \text{ and } \E_t[\z_t] = y_t\x_t.
\end{equation}

Next, notice that in Step 9, the algorithm runs the standard Hedge-algorithm update (see, for example, Section 2.1 in \citep{mwsurvey}) on ${d \choose k}$ experts, one for each subset of $[d]$ of size $k$, where, in round $t$, the cost 
of the expert corresponding to subset $S$ is defined to be\footnote{Recall that the costs in Hedge may be chosen adaptively.}
\begin{equation} \label{eq:cost}
\w_{S, t}^\top \X_t \w_{S, t} - 2\z_t^\top \w_{S, t} + y_t^2.
\end{equation}
It is easy to check, using the facts that $\|\x_t\| \leq 1$, $\|\w_{S, t}\| \leq 1$ and $p \geq q$, that the cost 
(\ref{eq:cost}) is bounded deterministically in absolute value by $O(\frac{1}{q}) = O(\frac{d^2}{(k' - k)^2})$.
Let $\E_{D_t}[\cdot]$ denote the expectation over the random choice of $\hS_t$ from the distribution $D_t$ conditioned on all other randomness up to and including round $t$. Since there are ${d \choose k}\le d^k$ experts in the Hedge algorithm here, the standard regret bound for Hedge~\citep[Theorem
2.3]{mwsurvey} with the specified value of $\etaH$ implies that for any subset $S$ of $[d]$ of size $k$, using
$\ln {d \choose k}\le k \ln d$, we have
\begin{equation} \label{eq:hedge-regret}
	\sum_{t=1}^T \E_{D_t}[\w_{\hS_t, t} \X_t \w_{\hS_t, t} - 2\z_t^\top \w_{\hS_t, t} + y_t^2]\ \leq\ \sum_{t=1}^T (\w_{S, t}^\top \X_t \w_{S, t} - 2\z_t^\top \w_{S, t} + y_t^2) + O(\tfrac{d^2}{(k' - k)^2}\sqrt{k\ln(d)T}).
\end{equation}
Next, we note, using \eqref{eq:unbiased} and the fact that conditioned on the randomness prior to round $t$, $\w_{S, t}$ is completely determined, that
(for any $S$)
\begin{equation} \label{eq:hedge-unbiased}
	\E_t[\w_{S, t}^\top \X_t \w_{S, t} - 2\z_t^\top \w_{S, t} + y_t^2]\ =\ \w_{S, t}^\top \x_t\x_t^\top \w_{S, t} - 2y_t\x_t^\top \w_{S, t} + y_t^2\ =\ (\w_{S, t} \cdot \x_t - y_t)^2.
\end{equation}
Taking expectations on both sides of \eqref{eq:hedge-regret} over all the randomness in the algorithm, and using \eqref{eq:hedge-unbiased}, we get that for any subset $S$ of $[d]$ of size $k$, we have
\begin{equation} \label{eq:experts-regret}
	\sum_{t=1}^T \E[(\w_{\hS_t, t} \cdot \x_t - y_t)^2]\ \leq\ \sum_{t=1}^T \E[(\w_{S, t} \cdot \x_t - y_t)^2] + O(\tfrac{d^2}{(k' - k)^2}\sqrt{k\log(d)T}).
\end{equation}
The left-hand side of (\ref{eq:experts-regret}) equals $\sum_{t=1}^T \E[(\hy_t - y_t)^2]$. We now analyze the right-hand side.

For any given subset $S$ of $[d]$ of size $k$, we claim that in Step 10 of the algorithm, the parameter vector $\w_{S, t}$ is updated using stochastic gradient descent with the loss function $\ell_t(\w) := (\x_t^\top \I_S \w - y_t)^2$ over the set over $\{\w\ |\ ||\w||_2\le 1\}$, only on the coordinates in $S$, while the coordinates not in $S$ are fixed to $0$. To prove this claim, first, we note that the premultiplication by $\I_S$ in the update in Step 10 ensures that in the parameter vector $\w_{S, t+1}$ all coordinates that are not in $S$ are set to $0$, assuming that
coordinates of $\w_{S,t}$ not in $S$ were 0. 

Next, at time $t$, consider the loss function $\ell_t(\w) = (\x_t^\top \I_S \w - y_t)^2$. The gradient of this loss function at $\w_{S, t}$ is 
\[\nabla \ell_t(\w_{S, t})\ =\ 2(\x_t^\top \I_S \w_{S, t} - y_t) \I_S \x_t\ =\ 2\I_S(\x_t\x_t^\top \w_{S, t} - y_t \x_t),\] 
where we use the fact that $\I_S \w_{S, t} = \w_{S, t}$ since $\w_{S, t}$ has zeros in coordinates not in $S$. Now, by \eqref{eq:unbiased}, we have
\[ \E_t[2\I_S(\X_t \w_{S, t} - \z_t)]\ =\ 2\I_S(\x_t\x_t^\top \w_{S, t} - y_t \x_t),\]
and thus, Step 10 of the algorithm is a stochastic gradient descent update as claimed.
Furthermore, a calculation similar to the one for bounding the loss of the experts in the Hedge algorithm shows that the norm of the stochastic gradient is bounded deterministically by $O(\frac{1}{q})$,
which is  $O(\frac{d^2}{(k' - k)^2})$.

Using a standard regret bound for stochastic gradient descent (see, for example, Lemma 3.1 in \citep{FKM}) with the specified value of $\etaGD$, we conclude that for any fixed vector $\w$ of $\ell_2$ norm at most $1$, we have,
\[ \sum_{t=1}^T \E[(\x_t^\top \I_S \w_{S, t} - y_t)^2]\ \leq\ \sum_{t=1}^T (\x_t^\top \I_S \w - y_t)^2 + O(\tfrac{d^2}{(k' - k)^2}\sqrt{T}).\]
Since $\I_S \w_{S, t} = \w_{S, t}$, the left hand side of the above inequality equals $\sum_{t=1}^T \E[(\w_{S, t} \cdot \x_t - y_t)^2]$.


Finally, let $\w$ be an arbitrary $k$-sparse vector of $\ell_2$ norm at most $1$. Let $S = \{i\ |\ w_i \neq 0\}$. Note that $|S| \leq k$, and $\I_S(\w)=\w$. Applying the above bound for this set $S$, we get
\begin{equation} \label{eq:sgd-regret}
	\sum_{t=1}^T \E[(\w_{S, t} \cdot \x_t - y_t)^2]\ \leq\ \sum_{t=1}^T (\w \cdot \x_t - y_t)^2 + O(\tfrac{d^2}{(k' - k)^2}\sqrt{T}).
\end{equation}
Combining the inequality (\ref{eq:sgd-regret}) with inequality (\ref{eq:experts-regret}), we conclude that
\[
\sum_{t=1}^T \E[(\hy_t - y_t)^2]\ \leq\ \sum_{t=1}^T (\w \cdot \x_t - y_t)^2 + O(\tfrac{d^2}{(k' - k)^2}\sqrt{k \log(d)T}).
\]
This gives us the required regret bound.
\end{proof}

\section{Computational lower bound}

In this section we show that there cannot exist an efficient no-regret algorithm for the online sparse regression problem unless $\NP \subseteq \BPP$. This hardness result follows from the hardness of approximating the \setcover problem. We give a reduction showing that if there were an efficient no-regret algorithm $\algosr$ for the online sparse regression problem, then we could distinguish, in randomized polynomial time, between two instances of the \setcover problem: in one of
which there is a small set cover, and in the other of which any set cover is large. This task is known to be \NP-hard for specific parameter values. Specifically, our reduction has the following properties:
\begin{enumerate}
 	\item If there is a small set cover, then in the induced online sparse regression problem there is a $k$-sparse parameter vector (of
$\ell_2$ norm at most 1) giving $0$ loss, and thus the algorithm $\algosr$ must have small total loss (equal to the regret) as well. 

 	\item If there is no small set cover, then the prediction made by $\algosr$ in any round has at least a constant loss in expectation, which implies that its total (expected) loss must be large, in fact, linear in $T$.
 \end{enumerate}
 By measuring the total loss of the algorithm, we can distinguish between the the two instances of the \setcover problem mentioned above with probability at least $3/4$, thus yielding a \BPP algorithm for an \NP-hard problem.

The starting point for our reduction is the work of \citet{DinurSteurer} who give a polynomial-time reduction of deciding satisfiability of \CNF formulas to distinguishing instances of \setcover with certain useful combinatorial properties. We denote the satisfiability problem of \CNF formulas by \SAT.

\begin{reduction} \label{reduction:dinur-steurer}
	For any given constant $D > 0$, there is a constant $c_D \in (0, 1)$ and a $poly(n^D)$-time algorithm that takes a $\CNF$ formula $\phi$ of size $n$ as input and constructs a \setcover instance over a ground set of size $m = \poly(n^D)$ with $d = \poly(n)$ sets, with the following properties:
\begin{enumerate}
	\item if $\phi \in \SAT$, then there is a collection of $k = O(d^{c_D})$ sets, which covers each element {\em exactly once}, and
	\item if $\phi \notin \SAT$, then no collection of $k' = \lfloor D\ln(d) k\rfloor $ sets covers all elements; i.e., at least one element is left uncovered.
\end{enumerate}
The \setcover instance generated from $\phi$ can be encoded as a binary matrix $\M_\phi \in \{0, 1\}^{m \times d}$ with the rows corresponding to the elements of the ground set, and the columns correspond to the sets, such that each column is the indicator vector of the corresponding set.
\end{reduction}

Using this reduction, we now show how an efficient, no-regret algorithm for online sparse regression can be used to give a \BPP algorithm for \SAT.

\begin{algorithm}[t]
\caption{\textsc{Algorithm $\algsat$ for deciding satisfiability of \CNF formulas}
\label{algorithm:sat-decider}}
\begin{algorithmic}[1]
\REQUIRE A constant $D > 0$, and an algorithm $\algosr$ for the $(k, k', d)$-online sparse regression problem with $k = O(d^{c_D})$, where $c_D$ is the constant from Reduction~\ref{reduction:dinur-steurer}, and $k' = \lfloor D \ln(d) k \rfloor$, that runs in $\poly(d, T)$ time per iteration with regret bounded by $p(d) \cdot T^{1-\delta}$ with probability at least $3/4$.
\REQUIRE A \CNF formula $\phi$.
\STATE Compute the matrix $\M_\phi$ and the associated parameters $k, k', d, m$ from Reduction~\ref{reduction:dinur-steurer}.
\STATE Run $\algosr$ with the parameters $k, k', d$ for $T := \lceil \max\{(2p(d) mdk)^{1/\delta}, 256m^2d^2k^2\}\rceil$ iterations.
\FOR{$t=1, 2, \ldots, T$}
\STATE Sample a row of $\M_\phi$ uniformly at random; call it $\hx_t$.
\STATE Sample $\sigma_t \in \{-1, 1\}$ uniformly at random independently of $\hx_t$.
\STATE Set $\x_t = \frac{\sigma_t}{\sqrt{d}} \hx_t$ and $y_t = \frac{\sigma_t}{\sqrt{dk}}$.
\STATE Obtain a set of coordinates $S_t$ of size at most $k'$ by running $\algosr$, and provide it the coordinates $\x_t(S_t)$.
\STATE Obtain the prediction $\hy_t$ from $\algosr$, and provide it the true label $y_t$.
\ENDFOR
\IF{$\sum_{t=1}^T (y_t - \hy_t)^2 \leq \frac{T}{2mdk}$}
\STATE Output ``satisfiable''.
\ELSE
\STATE Output ``unsatisfiable''.
\ENDIF
\end{algorithmic}
\end{algorithm}

\begin{theorem} \label{thm:hardness} Let $D > 0$ be any given constant. Suppose there is an algorithm, $\algosr$, for the $(k, k', d)$-online sparse regression problem with $k = O(d^{c_D})$, where $c_D$ is the constant from Reduction~\ref{reduction:dinur-steurer}, and $k' = \lfloor D \ln(d) k\rfloor $, that runs in $\poly(d, T)$ time per iteration and has expected regret bounded by $\poly(d) T^{1-\delta}$ for some constant $\delta > 0$. Then $\NP \subseteq \BPP$.
\end{theorem}
\begin{proof}
Since the expected regret of $\algosr$ is bounded by $p(d) T^{1-\delta}$ (where $p(d)$ is a polynomial function of
$d$), by Markov's inequality we conclude that with probability at least $3/4$, the regret of $\algosr$ is bounded by $p(d) \cdot T^{1-\delta}$. Figure~\ref{algorithm:sat-decider} gives a randomized algorithm, $\algsat$, for deciding satisfiability of a given \CNF formula $\phi$ using the algorithm $\algosr$. Note that the feature vectors (i.e., the $\x_t$ vectors) generated by $\algsat$ are bounded in $\ell_2$ norm by $1$, as required. It is clear that $\algsat$ is a polynomial-time algorithm since $T$ is a polynomial function of $n$ (since $m, k, d, p(d)$ are polynomial functions of $n$), and $\algosr$ runs in $\poly(d, T)$ time per iteration.

We now claim that this algorithm correctly decides satisfiability of $\phi$ with probability at least $3/4$, and is hence a \BPP algorithm for $\SAT$. 

To prove this, suppose $\phi \in \SAT$. Then, there are $k$ sets in the \setcover which cover all elements with each element being covered exactly once. Consider the $k$-sparse parameter vector $\w$ which has $\frac{1}{\sqrt{k}}$ in the positions corresponding to these $k$ sets and $0$ elsewhere. Note that $\|\w\| \leq 1$, as required. Note that since this
collection of $k$ sets covers each element exactly once, we have $\M_\phi \w = \frac{1}{\sqrt{k}}\one$, where $\one$ is the all-1's vector. In particular, since $\hx_t$ is a row of $\M_\phi$, we have 
\[\w\cdot \x_t\ =\ \w \cdot \left(\frac{\sigma_t}{\sqrt{d}} \hx_t\right)\ =\ \frac{\sigma_t}{\sqrt{dk}}\ =\ y_t.\] 
Thus, $(\w \cdot \x_t - y_t)^2 = 0$ for all rounds $t$. Since algorithm $\algosr$ has regret at most $p(d) \cdot T^{1-\delta}$ with probability at least $3/4$, its total loss $\sum_{t=1}^T (\hy_t - y_t)^2$ is bounded by $p(d) \cdot T^{1-\delta} \leq \frac{T}{2mdk}$ (since $T \geq (2p(d) mdk)^{1/\delta}$) with probability at least $3/4$. Thus, in this case algorithm $\algsat$ correctly outputs ``satisfiable'' with probability at least $3/4$.

Next, suppose $\phi \notin \SAT$. Fix any round $t$ and let $S_t$ be the set of coordinates of size at most $k'$ selected by algorithm $\algosr$ to query. This set $S_t$ corresponds to $k'$ sets in the \setcover instance. Since $\phi \notin \SAT$, there is at least one element in the ground set that is not covered by any set among these $k'$ sets. This implies that there is at least one row of $\M_\phi$ that is $0$ in all the coordinates in $S_t$. Since $\hx_t$ is a uniformly random row of $\M_\phi$ chosen independently of $S_t$, we have
\[ \Pr[\x_t(S_t) = \zero]\ =\ \Pr[\hx_t(S_t) = \zero]\ \geq\ \frac{1}{m}.\]
Here, $\zero$ denotes the all-zeros vector of size $k'$. 

Now, we claim that $\E[y_t \hat y_t\ |\ \x_t(S_t) = \zero] = 0$. This is because the condition that $\x_t(S_t) = 0$ is equivalent to the
condition that $\hx_t (S_t) = 0$.  Since $y_t$ is chosen from $\{-\frac{1}{\sqrt{dk}}, \frac{1}{\sqrt{dk}}\}$ uniformly at random independently of $\hx_t$ and $\hy_t$, the claim follows. The expected loss of the online algorithm in round $t$ can now be bounded as follows:
\[\E[\left . (\hy_t - y_t)^2\ \right |\ \x_t(S_t) = \zero]\ =\ \E\left[\left . \hy_t^2 + \frac{1}{dk} - 2y_t\hat y_t\ \right |\ \x_t(S_t) = \zero\right]\]
\[ =\ \E\left[\left
. \hat y_t^2 + \frac{1}{dk}\  \right |\ \x_t(S_t) = \zero\right]\ \geq\ \frac{1}{dk},\]
and hence
\[\E[(y_t - \hat{y}_t)^2]\ \geq\ \E[(y_t - \hat{y}_t)^2\ |\ \x_t(S_t) = \zero]\cdot \Pr[\x_t(S_t) = \zero]\ \geq\ \frac{1}{dk}\cdot \frac{1}{m}\ =\ \frac{1}{mdk}.\] 
Let $\E_t[\cdot]$ denote expectation of a random variable conditioned on all randomness prior to round $t$. Since the choices of $\x_t$ and $y_t$ are independent of previous choices in each round, the same argument also implies that $\E_t[(y_t - \hat{y}_t)^2]\ \geq\ \frac{1}{mdk}$. Applying Azuma's inequality (see Theorem 7.2.1 in \citep{AlonSpencer}) to the martingale difference sequence $\E_t[(y_t - \hat{y}_t)^2] - (y_t - \hat{y}_t)^2$ for $t = 1, 2, \ldots, T$, since each term is bounded in absolute value by $4$, we get
\[ \Pr\left[\sum_{t=1}^T
\E_t[(y_t - \hat{y}_t)^2]
-(y_t - \hat{y}_t)^2 
 \geq 8\sqrt{T}\right]\ \leq\ 
\exp\left (-\tfrac{64T}{2\cdot 16T}\right )\ \leq\ \frac{1}{4}. \]
Thus, with probability at least $3/4$, the total loss $\sum_{t=1}^T (\hy_t - y_t)^2$ is greater than $\sum_{t=1}^T
\E_t[(y_t - \hat{y}_t)^2] - 8\sqrt{T} \geq \frac{1}{mdk}T - 8\sqrt{T} \geq \frac{T}{2mdk}$ (since $T \geq 256m^2d^2k^2$). Thus in this case the algorithm correctly outputs ``unsatisfiable'' with probability at least $3/4$.
\end{proof}

\paragraph{Parameter settings for hard instances.} Theorem~\ref{thm:hardness} implies that for any given constant $D > 0$, there is a constant $c_D$ such that the parameter settings $k = O(d^{c_D})$, and $k' = \lfloor D \ln(d) k \rfloor$ yield hard instances for the online sparse regression problem. The reduction of \citet{DinurSteurer} can be ``tweaked''\footnote{This is accomplished by simply replacing the \textsf{Label Cover} instance they construct with polynomially many disjoint copies of the same instance.} so that the $c_D$ is arbitrarily close to $1$ for any constant $D$.

We can now extend the hardness results to the parameter settings $k = O(d^\epsilon)$ for any $\epsilon \in (0, 1)$ and $k' = \lfloor D \ln(d) k \rfloor$ either by tweaking the reduction of \citet{DinurSteurer} so it yields $c_D = \epsilon$ if $\epsilon$ is close enough to $1$, or if $\epsilon$ is small, by adding $O(d^{1/\epsilon})$ all-zeros columns to the matrix $\M_\phi$. The two combinatorial properties of $\M_\phi$ in Reduction~\ref{reduction:dinur-steurer} are clearly still satisfied, and the proof of Theorem~\ref{thm:hardness} goes through.

\section{Conclusions}

In this paper, we prove that minimizing regret in the online sparse regression problem is computationally hard even if the learner is allowed access to many more features than the comparator, a sparse linear regressor. We complement this result by giving an inefficient no-regret algorithm. 

The main open question remaining from this work is what extra assumptions can one make on the examples arriving online to make the problem tractable. Note that the sequence of examples constructed in the lower bound proof is i.i.d., so clearly stronger assumptions than that are necessary to obtain any efficient algorithms.

\bibliographystyle{plainnat}
\bibliography{budgeted}

\end{document}